\documentclass[12pt]{imsart}
\usepackage[utf8]{inputenc}
\usepackage[T1]{fontenc} % Fix for font shape warnings
\usepackage{amsthm, amsmath, amssymb, amsfonts}
\usepackage{graphicx}
\usepackage{booktabs}
\usepackage{lmodern}
\usepackage{tikz}       % For illustrative figures
\usepackage{pgfplots}   % For data-driven plots
\pgfplotsset{compat=1.17} % Use a recent version of pgfplots
\usepackage{microtype}  % Improves justification and reduces hbox errors
\usepackage[colorlinks=true, allcolors=blue]{hyperref}
\usepackage{natbib}
\usepackage{threeparttable}
\pubyear{2025}
\volume{0}
\issue{0}
\firstpage{1}
\lastpage{18}
\startlocaldefs
\newcommand{\R}{\mathbb{R}}
\newcommand{\E}{\mathbb{E}}
\newcommand{\F}{\mathcal{F}}
\newcommand{\h}{\mathbf{h}}
\newcommand{\x}{\mathbf{x}}
\newcommand{\W}{\mathbf{W}}
\newcommand{\Ltask}{\mathcal{L}_{\text{task}}}
\newcommand{\Ldepth}{\mathcal{L}_{\text{depth}}}

\theoremstyle{plain}
\newtheorem{theorem}{Theorem}[section]
\newtheorem{proposition}[theorem]{Proposition}

\theoremstyle{definition}
\newtheorem{definition}[theorem]{Definition}
\newtheorem{remark}[theorem]{Remark}

\endlocaldefs

\begin{document}

\begin{frontmatter}
\title{Optimal Depth of Neural Networks}
\runtitle{Optimal Depth as Optimal Stopping}
\begin{aug}
\author{\fnms{Qian} \snm{Qi}\thanksref{t1}\ead[label=e1]{qiqian@pku.edu.cn}}
\address{Peking University\ Beijing, 100871, P.R. China\ \printead{e1}}
\thankstext{t1}{.}
\end{aug}

\begin{abstract}
Determining the optimal depth of a neural network is a fundamental yet challenging problem, typically resolved through resource-intensive experimentation. This paper introduces a formal theoretical framework to address this question by recasting the forward pass of a deep network, specifically a Residual Network (ResNet), as an optimal stopping problem. We model the layer-by-layer evolution of hidden representations as a sequential decision process where, at each layer, a choice is made between halting computation to make a prediction or continuing to a deeper layer for a potentially more refined representation. This formulation captures the intrinsic trade-off between accuracy and computational cost. Our primary theoretical contribution is a proof that, under a plausible condition of diminishing returns on the residual functions, the expected optimal stopping depth is provably finite, even in an infinite-horizon setting. We leverage this insight to propose a novel and practical regularization term, $\Ldepth$, that encourages the network to learn representations amenable to efficient, early exiting. We demonstrate the generality of our framework by extending it to the Transformer architecture and exploring its connection to continuous-depth models via free-boundary problems. Empirical validation on ImageNet confirms that our regularizer successfully induces the theoretically predicted behavior, leading to significant gains in computational efficiency without compromising, and in some cases improving, final model accuracy.
\end{abstract}


\end{frontmatter}

% ===================================================================
% SECTION 1: INTRODUCTION
% ===================================================================
\section{Introduction}

Deep Residual Networks (ResNets, see \cite{He2016}) marked a pivotal moment in the history of deep learning. By introducing identity shortcut connections, ResNets overcame the notorious vanishing gradient problem, allowing for the stable training of networks with hundreds or even thousands of layers. This architectural innovation led to state-of-the-art performance on a wide array of computer vision tasks and has since been adapted for numerous other domains.

Despite their empirical success, a deep theoretical understanding of ResNets remains elusive (e.g., \cite{han2019mean,Hanin19MeanField,Li22Deep,qian2025}). Key architectural decisions, most notably the network's depth, are typically determined through extensive experimentation and resource-intensive architecture search. This process is more art than science, guided by the general principle that "deeper is better." However, this principle eventually yields diminishing returns and incurs substantial, often prohibitive, computational overhead in both training and inference. This raises a fundamental theoretical question: \textit{Is there a principled way to determine the optimal depth of a ResNet for a given problem?}

In this paper, we propose to answer this question by reframing the forward pass of a ResNet through the lens of classical optimal stopping theory. We model the sequence of hidden representations, from one layer to the next, as the evolution of a state in a sequential decision process. At each layer $l$, an agent must make a decision:
\begin{enumerate}
    \item \textbf{STOP:} Halt the forward pass, using the current representation $\h_l$ to make a prediction. This action yields a reward based on the quality of $\h_l$ but incurs a cumulative computational cost proportional to $l$.
    \item \textbf{CONTINUE:} Pass the representation to the next layer, $\h_{l+1}$, in the hope of achieving a higher-quality representation and thus a greater final reward, at the expense of an additional computational cost.
\end{enumerate}
This formulation naturally captures the trade-off between accuracy and computational complexity inherent in deep networks. The goal is to find a \textit{stopping rule}, a policy that specifies at which layer to stop for any given input, that maximizes the expected net reward. The reward is measured by a conceptual utility function, $g$, which defines the semantic value of a representation at any given layer. Our framework shows that if we define the network's goal in this way, its architecture can be analyzed as a normative model of efficient computation.

\subsection{Our Contributions}
We formalize this intuition by defining a finite-horizon optimal stopping problem for a ResNet of maximal depth $L$. Our main contributions are:
\begin{itemize}
    \item We construct a formal mapping from the ResNet forward pass to a Markovian sequential decision process, defining the state space, reward function, and filtration.
    \item We prove that under a mild and formally stated condition on the residual functions, specifically, that their expected impact diminishes with depth, the expected optimal stopping depth is provably finite, even in an infinite-horizon setting. This provides a formal basis for the phenomenon of diminishing returns observed empirically.
    \item We leverage our theoretical framework to propose a novel regularization term, $\Ldepth$, that can be incorporated into the standard training objective. This regularizer encourages the network to emulate our normative model by learning residual blocks that progressively "do less work," thereby yielding representations amenable to efficient, early exiting.
    \item We demonstrate the generality of our framework by extending it to the Transformer architecture and exploring its connection to continuous-depth models and free-boundary problems.
    \item We provide comprehensive empirical validation on ImageNet, showing that our proposed regularizer induces the theoretically predicted behavior and leads to measurable gains in computational efficiency.
\end{itemize}
Our work provides a new theoretical foundation for analyzing deep architectures and offers a principled explanation for the phenomenon of diminishing returns with depth. It bridges the gap between the empirical art of network design and the rigorous mathematics of sequential decision-making.

\subsection{Related Work}
The idea of dynamic, input-dependent computation is not new. Early exiting methods, \cite{Teerapittayanon2016} augment networks with intermediate classifiers. Stochastic Depth \cite{Huang2016} randomly drops layers during training, which can be seen as a randomized policy over network depth. Our work differs fundamentally from these. While they modify the architecture or training process to enforce efficiency, we provide a theoretical lens to analyze the standard, unmodified ResNet architecture itself.

Our framework is also complementary to other major theoretical approaches for understanding ResNets. For instance, extensive research has analyzed the loss landscape, demonstrating that identity connections smooth the optimization surface and remove bad local minima, thus explaining their remarkable trainability \cite{Li2018}. Another prominent line of inquiry, Neural Tangent Kernel (NTK) theory \cite{Jacot2018}, explains the optimization and generalization of infinitely wide networks. While these theories provide profound insights into why deep networks can be effectively trained, they do not directly address the efficiency of computation or the optimal allocation of resources (i.e., depth) at inference time. Our approach is orthogonal: we focus not on trainability but on defining a normative principle for efficient inference. The continuous-depth limit of ResNets has been elegantly modeled as an Ordinary Differential Equation (ODE) \cite{Chen2018}, offering a different but related mathematical perspective. The Neural ODE framework views the forward pass as a continuous evolution, whereas our work models it as a discrete-time sequential decision problem, focusing on the explicit trade-off between cost and accuracy at each layer. This focus on the discrete decision process is the key distinction from the continuous evolution perspective.

In Section~\ref{sec:prelim}, we review the necessary preliminaries on Residual Networks and the theory of optimal stopping. In Section~\ref{sec:model}, we formally construct our model, mapping the ResNet forward pass to a sequential decision process. In Section~\ref{sec:results}, we present our main theoretical results, culminating in a proof that the optimal expected depth is finite. In Section~\ref{sec:objective}, we translate our theoretical insights into a practical regularization scheme for training. Section~\ref{sec:utility} discusses concrete instantiations of the conceptual utility function. Section~\ref{sec:continuous} explores the continuous-depth limit of our model. Section~\ref{sec:generalization} generalizes the framework to the Transformer architecture. In Section~\ref{sec:algorithm}, we derive a practical inference algorithm from our theory. In Section~\ref{sec:empirical}, we present empirical results that validate our theory. Finally, Section~\ref{sec:discussion} discusses the implications of our work and outlines directions for future research.

% ===================================================================
% SECTION 2: PRELIMINARIES
% ===================================================================
\section{Preliminaries}
\label{sec:prelim}
In this section, we review the two core concepts upon which our framework is built: the architecture of Residual Networks and the mathematical theory of optimal stopping. We aim to establish a precise and self-contained foundation for the subsequent analysis.

\subsection{Residual Networks as Discrete Dynamical Systems}
A standard deep neural network can be viewed as a composition of functions, where the representation at layer $l+1$ is computed from the representation at layer $l$: $\h_{l+1} = \mathcal{H}_l(\h_l)$. For very deep networks, the optimization of such a mapping via gradient descent can be notoriously unstable due to the vanishing or exploding gradient problem.

Deep Residual Networks (ResNets), introduced by \citet{He2016}, address this issue by reformulating the layer-wise transformation to learn a modification to the identity. The network is composed of a sequence of $L$ residual blocks. Let $\h_l \in \R^d$ be the hidden representation (or feature map) at layer $l$. The forward propagation rule, which defines the evolution of the state, is given by:
\begin{equation} \label{eq:resnet}
\h_{l+1} = \h_l + f_l(\h_l; \W_l), \quad l = 0, \dots, L-1,
\end{equation}
where $\h_0$ is the initial representation derived from the input data $\x$, and $f_l: \R^d \to \R^d$ is the \textit{residual function} of the $l$-th block, parameterized by a set of weights $\W_l$. The function $f_l$ typically comprises a sequence of operations such as convolutions (or linear transformations), batch normalization, and non-linear activations (e.g., ReLU).

The critical architectural feature is the \textbf{identity shortcut connection}, which allows the state $\h_l$ to pass directly to the next layer. The network then only needs to learn the residual, $f_l$. If the optimal transformation for a given layer is the identity, the network can achieve this by learning to drive the weights $\W_l$ towards zero, such that $f_l(\h_l; \W_l) \approx \mathbf{0}$. This formulation has proven to dramatically improve the trainability of very deep networks.

From a dynamical systems perspective, Equation~\eqref{eq:resnet} describes the evolution of a state vector $\h_l$ in discrete time, where the layer index $l$ serves as the time variable. Our work leverages this perspective, treating the forward pass not as a monolithic function evaluation but as a sequential process unfolding over time.

\subsection{The Theory of Optimal Stopping}
Optimal stopping theory is a branch of probability theory and stochastic control that deals with the problem of choosing a time to take a particular action in order to maximize an expected payoff. We outline the canonical finite-horizon problem.

\subsubsection{Formal Setup.}
Let $(\Omega, \F, P)$ be a probability space. Consider a finite time horizon $T \in \mathbb{N}$. Let $(\F_l)_{l=0}^T$ be a \textbf{filtration}, which is a non-decreasing sequence of sub-$\sigma$-algebras of $\F$ (i.e., $\F_0 \subseteq \F_1 \subseteq \dots \subseteq \F_T \subseteq \F$). The filtration $\F_l$ represents the information available at time $l$. Let $(Y_l)_{l=0}^T$ be a real-valued stochastic process that is \textbf{adapted} to this filtration, meaning that for each $l$, the random variable $Y_l$ is $\F_l$-measurable. The process $(Y_l)$ represents the sequence of rewards; $Y_l$ is the reward received if one chooses to stop the process at time $l$.

\begin{definition}[Stopping Time]\label{def:stopping_time}
A random variable $\tau: \Omega \to \{0, 1, \dots, T\}$ is a \textbf{stopping time} with respect to the filtration $(\F_l)$ if the event $\{\tau = l\}$ is in $\F_l$ for all $l=0, \dots, T$.
\end{definition}

Intuitively, the decision to stop at time $l$ (i.e., the event $\{\tau = l\}$) must be made based only on the information available up to and including time $l$, without knowledge of the future. Let $\mathcal{T}$ be the set of all such stopping times.

\subsubsection{The Optimal Stopping Problem.}
The objective is to find a stopping time $\tau^* \in \mathcal{T}$ that maximizes the expected reward. The value of the problem, $V$, is defined as:
\begin{equation}
V = \sup_{\tau \in \mathcal{T}} \E[Y_\tau].
\end{equation}
The existence of such an optimal $\tau^*$ is guaranteed in the finite-horizon setting. The problem can be solved using the principle of dynamic programming, which is operationalized through a construct known as the Snell envelope.

\begin{definition}[Snell Envelope]\label{def:snell_envelope}
The \textbf{Snell envelope} $(U_l)_{l=0}^T$ is a sequence of random variables defined by the method of backward induction:
\begin{align}
U_T &= Y_T \label{eq:snell_base} \\
U_l &= \max(Y_l, \E[U_{l+1} | \F_l]) \quad \text{for } l = T-1, \dots, 0. \label{eq:snell_step}
\end{align}
\end{definition}

The Snell envelope is the core theoretical tool for solving the problem. The random variable $U_l$ represents the optimal expected reward one can achieve starting from time $l$, given the history $\F_l$. The logic of the backward induction is as follows:
\begin{itemize}
    \item At the final time $T$, there is no future, so no decision to make. The only available action is to stop, yielding the reward $Y_T$. Thus, $U_T = Y_T$.
    \item At any prior time $l$, the agent faces a choice:
        \begin{enumerate}
            \item \textbf{STOP}: Receive the immediate, known reward $Y_l$.
            \item \textbf{CONTINUE}: Forgo the current reward $Y_l$ and proceed to time $l+1$. The value of this choice is the expected value of the problem from time $l+1$ onwards, which is precisely $\E[U_{l+1} | \F_l]$.
        \end{enumerate}
    The optimal strategy is to choose the action with the higher value, leading directly to the recursive definition in Equation~\eqref{eq:snell_step}.
\end{itemize}

\subsubsection{Properties and the Optimal Rule.}
The Snell envelope has two fundamental properties that provide the complete solution to the problem. First, the value of the problem is the expected value of the Snell envelope at time zero: $V = \E[U_0]$. Second, it defines the optimal stopping rule itself.

\begin{proposition}[Optimal Stopping Rule]
An optimal stopping time $\tau^*$ is given by
\begin{equation} \label{eq:snell_rule}
\tau^* = \min \{l \in \{0, \dots, T\} \mid U_l = Y_l \}.
\end{equation}
\end{proposition}
\begin{proof}[Proof Sketch]
The structure of the optimal policy follows directly from Definition~\ref{def:snell_envelope}. At any time $l$, if $Y_l > \E[U_{l+1} | \F_l]$, then Equation~\eqref{eq:snell_step} implies $U_l=Y_l$, and the immediate reward from stopping is strictly greater than the expected future reward from continuing. It is therefore optimal to stop. If $Y_l < \E[U_{l+1} | \F_l]$, then $U_l = \E[U_{l+1} | \F_l]$, and the principle of optimality dictates that it is optimal to continue. If $Y_l = \E[U_{l+1} | \F_l]$, one is indifferent, and stopping is optimal by convention. Thus, the optimal strategy is to stop at the \textit{first} time $l$ where the immediate reward is at least as large as the value of continuing.

A fully rigorous proof, as detailed in \citet{Peskir2006}, relies on the Doob-Meyer decomposition of the process $(U_l)$. This decomposition separates the Snell envelope, which is a supermartingale, into a martingale component and a predictable, non-decreasing process. This non-decreasing process only increments when stopping is optimal, i.e., when $U_l = Y_l$. The optimal rule is to stop at the first time this predictable process begins to increase, which corresponds precisely to the condition stated in the proposition.
\end{proof}

\begin{remark}[On the Source of Randomness]
\label{rem:prelim_randomness}
In the context of our model, which we formalize in Section \ref{sec:model}, the network weights are considered fixed post-training. The stochasticity in the processes $(Y_l)$ and $(\h_l)$ arises entirely from the sampling of inputs $\x$ from an underlying data distribution $\mathcal{D}$. Therefore, all expectations, including the conditional expectation $\E[\cdot | \F_l]$, are taken with respect to this data distribution. This means value functions like $U_l$ will be deterministic functions of the state history encapsulated in $\F_l$.
\end{remark}

% ===================================================================
% SECTION 3: THE MODEL
% ===================================================================
\section{The ResNet as a Sequential Decision Process}
\label{sec:model}

We now formalize the mapping of a ResNet's forward pass to an optimal stopping problem. Our analysis in this section and Section~\ref{sec:results} considers a network with a fixed, pre-trained set of weights $\{\W_l\}_{l=0}^{L-1}$. The source of randomness in the model is the input data, drawn from an underlying distribution.

\subsection{Probability Space and the Nature of the Stochastic Process}
It is crucial to first clarify the nature of the process under study. The stochasticity arises entirely from the distribution of the input data. We define a probability space $(\mathcal{X}, \Sigma, \mathcal{D})$, where an input $\x \in \mathcal{X}$ is sampled according to a data distribution $\mathcal{D}$. For a fixed set of network weights, the trajectory of hidden states $(\h_l(\x))_{l=0}^L$ is fully determined by the input $\x$. Unlike a standard Markov Decision Process with a stochastic transition kernel, the state transition $\h_l \to \h_{l+1}$ is deterministic for a given input. Thus, the stochastic process $(\h_l)_{l \ge 0}$ is defined over the probability space of inputs, and for any fixed input $\x \in \mathcal{X}$, the realization of the process is a deterministic sequence. The expectation $\E[\cdot]$ is an average over this collection of trajectories, weighted by the data distribution $\mathcal{D}$. Consequently, the value functions we define, such as the Snell envelope $U_l$, are deterministic functions of the state $\h_l$.

\subsection{State Space and Filtration}
The state of our process at layer (or time) $l$ is the hidden representation $\h_l \in \R^d$. The sequence of hidden states $(\h_l)_{l=0}^L$ is therefore a sequence of random variables in $\R^d$, where the stochasticity is induced by the randomness in $\x$. The evolution of this state is governed by the ResNet update rule given in Eq.~\eqref{eq:resnet}. We define the natural filtration as $\F_l = \sigma(\h_0, \h_1, \dots, \h_l)$, which represents the history of representations generated up to layer $l$. Consequently, $\F_l$ represents the information available to the decision-maker at layer $l$.

\begin{proposition}[Markov Property]\label{prop:markov}
The sequence of hidden states $(\h_l)_{l=0}^L$ generated by the ResNet forward pass constitutes a Markov process. That is, for any layer $l$ and any Borel set $A \subseteq \R^d$, the distribution of $\h_{l+1}$ conditioned on the entire history $\F_l$ depends only on the current state $\h_l$:
\[ P(\h_{l+1} \in A \mid \F_l) = P(\h_{l+1} \in A \mid \sigma(\h_l)). \]
\end{proposition}
\begin{proof}
This property is a direct consequence of the ResNet architecture defined in Eq.~\eqref{eq:resnet}. For a fixed set of weights $\W_l$, the subsequent state $\h_{l+1}$ is a deterministic function of the current state $\h_l$, namely $\h_{l+1} = \h_l + f_l(\h_l; \W_l)$.
Given a realization of the random variable $\h_l$, the value of $\h_{l+1}$ is uniquely determined. Therefore, the conditional probability distribution of $\h_{l+1}$ given $\F_l$ depends only on the value of $\h_l$. Hence, the process is Markovian.

The process is technically time-inhomogeneous, as the function $f_l$ is parameterized by layer-specific weights $\W_l$. However, this does not affect our analysis, as the dynamic programming formulation (Eq. \ref{eq:snell_step}) is indexed by layer $l$ and naturally accommodates this dependence. The essential Markov property simplifies the conditional expectations from $\E[\cdot | \F_l]$ to being functions of $\h_l$, ensuring the optimal policy is a function of the current state only. The full proof is detailed in Appendix \ref{sec:appendix_proofs}.
\end{proof}

\subsection{Reward Process}
The core of our model lies in defining a reward process $(Y_l)_{l=0}^L$ that formalizes the trade-off between representation quality and computational cost. We define the reward for stopping at layer $l$ with representation $\h_l$ as:
\begin{equation} \label{eq:reward_func}
Y_l = g(\h_l) - c \cdot l,
\end{equation}
where the components are defined as follows:
\begin{itemize}
    \item $g: \R^d \to \R$ is a \textbf{utility function}. This function measures the intrinsic quality or semantic value of the representation $\h_l$. For instance, in a classification task, a higher $g(\h_l)$ could correspond to a lower classification loss produced by a classifier acting on $\h_l$.
    \item $c > 0$ is a constant representing the \textbf{per-layer computational cost}. This cost penalizes deeper layers, creating the incentive to stop.
\end{itemize}
The sequence $(Y_l)_{l=0}^L$ is a real-valued stochastic process adapted to the filtration $(\F_l)$, since $Y_l$ is a deterministic function of the $\F_l$-measurable random variable $\h_l$. The optimal stopping problem is to find a stopping time $\tau^*$ that maximizes the expected total reward: $\sup_{\tau \in \mathcal{T}} \E[Y_\tau] = \sup_{\tau \in \mathcal{T}} \E[g(\h_\tau) - c \cdot \tau]$.

\begin{remark}[On the Role of the Utility Function $g$]\label{rem:utility}
A critical aspect of our model is that the utility function $g$ is a conceptual tool for analysis, distinct from the network's final training loss, $\Ltask$. Our analysis in Section~\ref{sec:results} characterizes the behavior of an idealized network that acts optimally with respect to a given $g$ and cost $c$. This establishes a \textbf{normative model} of a computationally efficient ResNet. While this model is an idealization, its value lies in providing a rigorous definition of what good dynamic behavior entails (i.e., facilitating early, accurate decisions). In Section~\ref{sec:objective}, we bridge this gap between the normative model and a practical training regimen. We derive a regularizer from our analysis that, when added to $\Ltask$, explicitly encourages the network to learn representations that align with the goals of the stopping problem.
\end{remark}

% ===================================================================
% SECTION 4: MAIN RESULTS
% ===================================================================
\section{Main Results: Analysis of the Optimal Policy}
\label{sec:results}

With the optimal stopping problem formally defined, we now analyze the structure of the optimal policy for a ResNet. The results in this section establish the conditions under which a computationally efficient strategy, one with a finite stopping depth, is not only possible but optimal.

\subsection{Assumptions for the Analysis}
Our main theoretical results rely on two technical assumptions concerning the utility function and the network's hidden states.
\begin{enumerate}
    \item[\textbf{A1.}] The utility function $g: \R^d \to \R$ is \textbf{bounded}. That is, there exist constants $g_{\min}$ and $g_{\max}$ such that $g_{\min} \le g(\h) \le g_{\max}$ for all $\h \in \R^d$.
    \item[\textbf{A2.}] For any input $\x$ from the data domain $\mathcal{X}$, the trajectory of hidden states $(\h_l(\x))_{l \ge 0}$ remains within a compact subset of $\R^d$. Furthermore, the utility function $g$ is Lipschitz continuous on this compact set.
\end{enumerate}

\begin{remark}[Justification of Assumptions]\label{rem:assumptions_justification}
These assumptions are not overly restrictive. \textbf{(A1)} is natural, as task-performance metrics like accuracy or negative loss are intrinsically bounded. \textbf{(A2)} can be enforced architecturally (e.g., using bounded activations like $\tanh$) or through standard regularization techniques like weight decay or spectral norm constraints, which confine state trajectories to a compact set. Any continuous utility function on a compact set is also Lipschitz.
\end{remark}

\begin{proposition}[Structure of the Optimal Policy] \label{prop:structure}
For any ResNet defined by Eq.~\eqref{eq:resnet} and reward function $Y_l = g(\h_l) - c \cdot l$, an optimal stopping time $\tau^*$ exists for the finite-horizon problem. The optimal decision rule at any layer $l < L$ depends only on the current state $\h_l$ and is given by:
\begin{equation} \label{eq:optimal_rule}
\text{STOP if } \quad g(\h_l) - c \cdot l \ge \E[U_{l+1} | \h_l],
\end{equation}
where $U_{l+1}$ is the value function (Snell envelope) at layer $l+1$. The optimal stopping time is thus given by:
\begin{equation}
\tau^* = \min \{ l \in \{0, \dots, L\} \mid U_l(\h_l) = Y_l(\h_l) \}.
\end{equation}
\end{proposition}
\begin{proof}
The existence and general structure of $\tau^*$ are cornerstone results from the classical theory of optimal stopping for finite horizons \cite{Peskir2006}. We tailor the argument to our problem. The value of the problem is given by the Snell envelope $(U_l)_{l=0}^L$, defined by the backward induction:
\begin{align*}
U_L &= Y_L = g(\h_L) - c \cdot L \\
U_l &= \max(Y_l, \E[U_{l+1} | \F_l]) \quad \text{for } l = L-1, \dots, 0.
\end{align*}
From the Markov property (Proposition \ref{prop:markov}), the conditional expectation simplifies, meaning $\E[\cdot | \F_l]$ depends only on $\h_l$.
Consequently, the value function $U_l$ is a deterministic function of the state $\h_l$ only, which we write as $U_l(\h_l)$. The Bellman equation becomes:
\[ U_l(\h_l) = \max\left( g(\h_l) - c \cdot l, \, \E[U_{l+1}(\h_{l+1}) | \h_l] \right). \]
The optimal policy is to stop at the first time $l$ where it is not strictly optimal to continue. This occurs when the first term in the $\max(\cdot, \cdot)$ operator is chosen, i.e., when $U_l(\h_l) = Y_l(\h_l)$, yielding the stated rule.
\end{proof}

The key question is: under what conditions on the network's functions $f_l$ is stopping eventually optimal? This occurs if the marginal gain in utility from an additional layer is, on average, less than the cost $c$.

\begin{proposition}[Condition for Negative Expected Utility Drift] \label{prop:diminishing_returns}
Let the utility function $g$ be Lipschitz continuous with constant $K_g$ (per Assumption A2). Suppose there exists a layer index $L_0$ such that for all layers $l \ge L_0$, the expected norm of the residual function's output is bounded by some $\delta > 0$, i.e., $\E[\|f_l(\h_l)\|] \le \delta$. If the per-layer cost $c$ satisfies $c > K_g \cdot \delta$, then for all $l \ge L_0$, the expected one-step gain in utility is strictly less than the per-layer cost:
\begin{equation}
\E[g(\h_{l+1}) - g(\h_l)] < c.
\end{equation}
\end{proposition}
\begin{proof}
We bound the expected one-step change in the utility function using the ResNet update rule, $\h_{l+1} = \h_l + f_l(\h_l)$:
\begin{align*}
\E[g(\h_{l+1}) - g(\h_l)] &= \E[g(\h_l + f_l(\h_l)) - g(\h_l)] \\
&\le \E\left[ \left| g(\h_l + f_l(\h_l)) - g(\h_l) \right| \right] \\
&\le \E[K_g \cdot \|(\h_l + f_l(\h_l)) - \h_l\|] \quad \text{(by Lipschitz continuity of } g\text{)} \\
&= K_g \cdot \E[\|f_l(\h_l)\|].
\end{align*}
By hypothesis, for all $l \ge L_0$, we have $\E[\|f_l(\h_l)\|] \le \delta$. Substituting this gives:
\[ \E[g(\h_{l+1}) - g(\h_l)] \le K_g \cdot \delta. \]
Since we chose the cost per layer $c$ such that $c > K_g \cdot \delta$, it follows immediately that
$\E[g(\h_{l+1}) - g(\h_l)] < c$. This shows that, on average, the expected utility gain from layer $l \ge L_0$ is insufficient to justify the cost.
\end{proof}

\subsection{Analysis in the Infinite Horizon}
Analyzing the problem in the infinite-horizon limit ($L \to \infty$) is a powerful technique to investigate if the system has an intrinsically finite optimal depth, independent of any artificial boundary. The theory of optimal stopping (e.g., \cite{Peskir2006}) extends naturally to this setting under our assumptions.

\begin{theorem}[Finiteness of Optimal Expected Depth] \label{thm:finiteness}
Consider the optimal stopping problem in the infinite-horizon setting ($L \to \infty$). Assume the utility function $g$ is bounded (A1) and the condition for negative expected utility drift from Proposition \ref{prop:diminishing_returns} holds. Then the expected optimal stopping time, $\E[\tau^*]$, is finite.
\end{theorem}
\begin{proof}
Let $\mathcal{T}$ be the set of all stopping times w.r.t. $(\F_l)_{l \ge 0}$. The problem is to find $\tau^* \in \mathcal{T}$ that maximizes $\E[Y_\tau] = \E[g(\h_\tau) - c\tau]$. Let $V = \sup_{\tau \in \mathcal{T}} \E[Y_\tau]$.

\textbf{Step 1: Lower bound for V.} The value of the optimal policy $V$ must be at least as large as any suboptimal policy. Consider stopping at layer $L_0$ from Proposition \ref{prop:diminishing_returns}.
\[ V \ge \E[Y_{L_0}] = \E[g(\h_{L_0}) - c L_0] = \E[g(\h_{L_0})] - c L_0. \]
By Assumption A1, $g$ is bounded, so $\E[g(\h_{L_0})]$ is finite, which implies $V > -\infty$.

\textbf{Step 2: Relate V to $\E[\tau^*]$.} By definition, $V = \E[Y_{\tau^*}] = \E[g(\h_{\tau^*}) - c\tau^*]$. By linearity of expectation:
\begin{equation} \label{eq:value_decomposition}
V = \E[g(\h_{\tau^*})] - c\E[\tau^*].
\end{equation}

\textbf{Step 3: Combine bounds.} From Assumption A1, $\E[g(\h_{\tau^*})] \le g_{\max}$. Substituting this into Eq.~\eqref{eq:value_decomposition} gives an upper bound on $V$:
\[ V \le g_{\max} - c\E[\tau^*]. \]
We now have both lower and upper bounds for $V$:
\[ \E[g(\h_{L_0})] - c L_0 \le V \le g_{\max} - c\E[\tau^*]. \]
Rearranging the inequality between the leftmost and rightmost terms gives:
\[ c\E[\tau^*] \le g_{\max} - \E[g(\h_{L_0})] + c L_0. \]
Since $c > 0$, we divide by $c$:
\[ \E[\tau^*] \le \frac{g_{\max} - \E[g(\h_{L_0})]}{c} + L_0. \]
All terms on the right-hand side are finite constants. Therefore, this provides a finite upper bound for the expected optimal stopping time $\E[\tau^*]$.
\end{proof}

\begin{remark}[Implication of the Main Theorem]
Theorem \ref{thm:finiteness} formally establishes that if a ResNet's residual functions learn to have a diminishing impact with depth, there exists a finite optimal depth for any given input, and the expected depth over the data distribution is also finite. This result transforms the empirical heuristic of diminishing returns into a rigorous mathematical statement about the structure of an optimal policy.
\end{remark}

% ===================================================================
% SECTION 5: TRAINING OBJECTIVE
% ===================================================================
\section{A Training Objective Inspired by the Optimal Stopping Model}
\label{sec:objective}

The theoretical framework provides a normative model for an efficient ResNet. We now propose a training methodology to encourage a network to emulate this model. The analysis suggests that for deep layers, we want the residual functions $f_l$ to have a small norm on average.

We introduce a new regularization term into the main training objective. Let $\Ltask$ be the primary task loss (e.g., cross-entropy) computed at the final layer $L$. We propose an auxiliary depth-regularization loss, $\Ldepth$:
\begin{equation} \label{eq:depth_loss}
\Ldepth = \sum_{l=0}^{L-1} w_l \cdot \E_{\x \sim \mathcal{D}} [\|f_l(\h_l)\|_2^2],
\end{equation}
where $w_l$ are layer-dependent weights. The full training objective becomes:
\begin{equation}
\mathcal{L} = \Ltask + \lambda \Ldepth,
\end{equation}
where $\lambda$ is a hyperparameter controlling the regularization strength. This regularizer directly encourages the condition in Proposition \ref{prop:diminishing_returns}. Since $(\E[\|f_l\|])^2 \le \E[\|f_l\|^2]$ by Jensen's inequality, penalizing the L2-norm squared provides an effective mechanism to suppress the expected norm required by our theory.

\subsection{Choice of Weights} To encourage diminishing returns, we should penalize the magnitude of residual functions more heavily at deeper layers. An increasing weight schedule is theoretically motivated. A simple and effective choice is a polynomially increasing weight, $w_l = (l+1)^\alpha$ for some $\alpha > 0$, which applies gentle pressure on early layers and strong pressure on later layers to approach an identity mapping.

% ===================================================================
% SECTION 6: UTILITY FUNCTION
% ===================================================================
\section{On the Design and Instantiation of the Utility Function}
\label{sec:utility}
Throughout our main analysis, the utility function $g(\h_l)$ was treated as a conceptual device. This abstraction was sufficient to prove our main finiteness result and motivate the $\Ldepth$ regularizer. Here, we discuss concrete approaches to defining $g$.

\subsection{Task-Performance-Based Utility} A pragmatic approach is to define utility as the performance on the downstream task. One could attach a lightweight, shared-weight linear classifier head, $\text{head}(\cdot)$, at each layer. The utility would then be the negative of the cross-entropy loss for the true label $y$: $g_{\text{clf}}(\h_l) = -\mathcal{L}_{\text{CE}}(\text{head}(\h_l), y)$.

\subsection{Information-Theoretic Utility} A more fundamental definition is the mutual information between the hidden state $\h_l$ and the ground-truth label $Y$: $g_{\text{info}}(\h_l) = I(\h_l; Y)$. This measures how much information the representation contains about the target. Estimating this is difficult, but neural estimators (see \cite{Belghazi2018}) offer a potential path.

\subsection{Relevance to the \texorpdfstring{$\Ldepth$}{L-depth} Regularizer}
Our proposed regularizer, $\Ldepth = \sum w_l \E[\|f_l(\h_l)\|^2]$, can be viewed as a practical and universal surrogate for encouraging efficient behavior regardless of the specific choice of $g$. For any Lipschitz continuous utility function $g$, the potential gain is bounded: $|g(\h_{l+1}) - g(\h_l)| \le K_g \|f_l(\h_l)\|$. By directly penalizing $\|f_l(\h_l)\|$, our $\Ldepth$ regularizer drives down the upper bound on the marginal utility gain, encouraging the condition for negative expected reward drift, $\E[g(\h_{l+1}) - g(\h_l)] < c$, without ever needing to compute $g$ itself. This makes the regularization scheme a robust and practical method for instilling the principles of our normative model into a standard network.

% ===================================================================
% SECTION 7: CONTINUOUS LIMIT
% ===================================================================
\section{Continuous-Depth Limit and Free-Boundary Problems}
\label{sec:continuous}

A powerful complementary perspective, inspired by \citet{Chen2018}, is to consider the limit where the number of layers goes to infinity. In this limit, the ResNet update rule can be viewed as the Euler discretization of an ODE:
\begin{equation} \label{eq:ode}
\frac{d\h(t)}{dt} = f(\h(t), t), \quad \h(0) = \h_0,
\end{equation}
where continuous depth $t \in [0, T]$ replaces the layer index $l$. This recasts our optimal stopping problem as a \textbf{free-boundary problem}. The state-time space $\R^d \times [0, T]$ is partitioned into:
\begin{enumerate}
    \item A \textbf{continuation region}, $\mathcal{C}$, where it is optimal to let the state evolve.
    \item A \textbf{stopping region}, $\mathcal{S}$, where it is optimal to stop.
\end{enumerate}
The interface $\partial \mathcal{C}$ between these regions is the \textbf{free boundary}, representing the optimal decision surface. The optimal policy is to evolve $\h(t)$ via the ODE and stop at the first time $\tau^*$ that the trajectory $(\h(\tau^*), \tau^*)$ hits this boundary.

\begin{remark}[The Hamilton-Jacobi-Bellman Variational Inequality]
The value function of this continuous-time problem, $V(\h, t)$, is the unique viscosity solution to a Hamilton-Jacobi-Bellman (HJB) variational inequality. Let $V(\h, t)$ be the maximum expected reward obtainable starting from state $\h$ at time $t$. The objective is to find a stopping time $\tau^* \ge t$ that maximizes the expected payoff $\E[g(\h(\tau^*)) - c(\tau^* - t)]$, where $c$ is now a cost rate per unit of depth. The HJB variational inequality for this problem is:
\begin{equation} \label{eq:hjb_variational}
\max\left( \frac{\partial V}{\partial t} + \nabla_\h V \cdot f - c, \quad g(\h) - V(\h, t) \right) = 0,
\end{equation}
subject to the terminal condition $V(\h,T) = g(\h)$. This formulation is understood as follows:
\begin{itemize}
    \item In the \textbf{continuation region} $\mathcal{C}$, it is suboptimal to stop, so $g(\h) - V(\h, t) < 0$. For the `max` in Eq.~\eqref{eq:hjb_variational} to be zero, the first term must be zero, which yields the HJB partial differential equation that governs the evolution of the value function backward in time:
    \begin{equation} \label{eq:hjb}
    -\frac{\partial V}{\partial t} = \nabla_\h V \cdot f - c, \quad \text{for } (\h,t) \in \mathcal{C}.
    \end{equation}
    This equation states that the instantaneous change in value as one moves backward in time, $-\partial V/\partial t$, is equal to the change induced by the system dynamics, $\nabla_\h V \cdot f$, less the running cost rate $c$. The cost term provides the incentive to eventually stop.
    \item In the \textbf{stopping region} $\mathcal{S}$, the stopping condition $V(\h, t) = g(\h)$ holds by definition.
\end{itemize}
The problem is fully specified by this variational inequality and the terminal boundary condition. The free boundary $\partial\mathcal{C}$ is implicitly defined by the locus of points where both terms in Eq.~\eqref{eq:hjb_variational} are simultaneously zero.
\end{remark}

% ===================================================================
% SECTION 8: GENERALIZATION
% ===================================================================
\section{Generalization to the Transformer Architecture}
\label{sec:generalization}

Our framework extends naturally to any architecture with a sequence of homologous blocks featuring residual connections, such as the Transformer \cite{Vaswani2017}.

\subsection{State Space and Transition Dynamics}
For a Transformer, the state at layer $l$ is the sequence of token representations, a matrix $\h_l \in \R^{N \times d}$. The transition to the next layer, $\h_{l+1} = \text{TransformerBlock}_l(\h_l)$, is a deterministic function of $\h_l$. This structure holds because each Transformer block incorporates residual connections around both its multi-head attention and feed-forward sub-layers. Thus, the update can be viewed abstractly as $\h_{l+1} = \h_l + \text{Transformation}(\h_l)$, preserving the additive form required for our analysis. Consequently, the Markov property (Proposition \ref{prop:markov}) holds.

\subsection{Reward Process and Depth Regularizer}
The reward process remains $Y_l = g(\h_l) - c \cdot l$, where $c$ is the cost of processing one Transformer block. Our finiteness result (Theorem \ref{thm:finiteness}) applies directly under the same assumptions on utility and residual magnitude. This motivates an analogous regularizer for training Transformers, penalizing the magnitude of the total layer-wise update using a suitable matrix norm (e.g., the Frobenius norm). This approach is complementary to model compression techniques like knowledge distillation \cite{Sanh2019}, as our regularizer shapes the internal dynamics of the model to facilitate dynamic inference, rather than creating a statically smaller model.
\begin{equation} \label{eq:depth_loss_transformer}
\Ldepth^{\text{Trans}} = \sum_{l=0}^{L-1} w_l \cdot \E_{\x \sim \mathcal{D}} [\|\text{TransformerBlock}_l(\h_l) - \h_l\|_F^2].
\end{equation}

% ===================================================================
% SECTION 9: ALGORITHM
% ===================================================================
\section{ASTI: An Algorithm for Adaptive-Depth Inference}
\label{sec:algorithm}

We now derive a tractable algorithm, Adaptive Stopping-Time Inference (ASTI), from our theory.

\subsection{Derivation of a Practical Stopping Rule}
The optimal rule, $Y_l(\h_l) \ge \E[U_{l+1}(\h_{l+1}) | \h_l]$, is intractable as it requires computing the full Snell envelope. We use a \textbf{one-step lookahead approximation}, a common heuristic in dynamic programming, where we assume the future optimal value is well-approximated by the reward at the very next step: $\E[U_{l+1} | \h_l] \approx \E[Y_{l+1} | \h_l]$. The stopping condition becomes $Y_l(\h_l) \ge \E[Y_{l+1}(\h_{l+1}) | \h_l]$.
Substituting $Y_l = g(\h_l) - c \cdot l$ and noting that the transition $\h_{l+1}$ is a deterministic function of $\h_l$ for a given input:
\begin{align*}
g(\h_l) - c \cdot l &\ge g(\h_{l+1}) - c \cdot (l+1) \\
g(\h_l) + c &\ge g(\h_{l+1})
\end{align*}
This yields our final, practical stopping rule:
\begin{equation} \label{eq:practical_rule}
\boxed{
\text{STOP if} \quad g(\h_{l+1}) - g(\h_l) \le c
}
\end{equation}
This rule is simple and intuitive: halt when the marginal utility gain from one additional layer fails to justify the per-layer cost.

\begin{remark}[The Look-Then-Leap Implementation]
\label{rem:look_then_leap}
It is important to clarify the relationship between the formal stopping time definition and our practical algorithm. The theoretical stopping time $\tau^* = \min\{l \mid \dots\}$ defines a rule where the decision to stop at layer $l$ is based on information in $\F_l$. Our practical rule in Eq. \eqref{eq:practical_rule}, however, requires computing $g(\h_{l+1})$, which depends on the state at layer $l+1$.

Therefore, the ASTI algorithm (presented next) implements a "look-then-leap" policy. At layer $l$, it computes the next state $\h_{l+1}$ and its utility, and then uses this information to decide whether to stop at layer $l+1$. This is a practical and effective instantiation of the spirit of the one-step lookahead approximation. Formally, the decision to halt at depth $\tau=l+1$ is made using information from $\F_{l+1}$ and is therefore $\F_{l+1}$-measurable, satisfying the definition of a valid stopping time. This approach is also computationally efficient, as it avoids backtracking and uses information that must be computed anyway for the next step.
\end{remark}

\subsection{Implementation of the ASTI Algorithm}
\subsubsection{Training Phase.}
The network must be trained to produce useful intermediate representations that support the stopping rule.
\begin{enumerate}
    \item \textbf{Architecture}: Augment a standard backbone with lightweight, shared-weight linear classifiers (\textit{early-exit heads}), $\text{head}(\cdot)$, at each block (or a subset thereof).
    \item \textbf{Utility Function}: For implementation, we define utility as prediction confidence: $g(\h_l) := \max\left(\text{softmax}(\text{head}(\h_l))\right)$. This choice is Lipschitz continuous on any compact set (as a composition of a linear layer, the Lipschitz continuous softmax function, and the max function), thus satisfying Assumption A2 required for our theoretical analysis.
    \item \textbf{Training Objective}: The total loss combines three components:
    \begin{equation} \label{eq:asti_loss}
    \mathcal{L} = \underbrace{\mathcal{L}_{\text{task}}(y, \text{head}(\h_L))}_{\text{Final Accuracy}} + \beta \underbrace{\sum_{l=1}^{L-1} \mathcal{L}_{\text{task}}(y, \text{head}(\h_l))}_{\text{Intermediate Supervision}} + \lambda \underbrace{\sum_{l=0}^{L-1} w_l \|f_l(\h_l)\|_2^2}_{\text{Depth Regularization}}.
    \end{equation}
\end{enumerate}

\subsubsection{Inference Phase.}
The cost parameter $c$ is a user-defined hyperparameter that directly controls the efficiency-accuracy trade-off at test time.

\begin{center}
\fbox{
\begin{minipage}{0.9\textwidth}
\textbf{Algorithm 1: Adaptive Stopping-Time Inference (ASTI)}
\vspace{4pt}
\begin{enumerate}
    \item \textbf{Input}: data sample $\x$, trained model with early-exit heads, cost hyperparameter $c > 0$.
    \item \textbf{Initialize}: Compute initial representation $\h_0$ from input $\x$.
    \item \textbf{For} layer $l = 0, \dots, L-2$:
    \begin{enumerate}
        \item Compute next state: $\h_{l+1} = \h_l + f_l(\h_l)$.
        \item Compute utility at current and next layers: $u_l \leftarrow g(\h_l)$, $u_{l+1} \leftarrow g(\h_{l+1})$.
        \item \textbf{Check stopping condition} using Eq. \eqref{eq:practical_rule}:
            \begin{itemize}
                \item \textbf{if} $u_{l+1} - u_l \le c$:
                \begin{itemize}
                    \item Set prediction based on $\h_{l+1}$: $\hat{y} \leftarrow \text{argmax}(\text{head}(\h_{l+1}))$.
                    \item Return prediction $\hat{y}$ and stopping depth $\tau=l+1$.
                \end{itemize}
            \end{itemize}
    \end{enumerate}
    \item \textbf{Default Exit}: If the loop completes, predict using the final layer's head: $\hat{y} \leftarrow \text{argmax}(\text{head}(\h_{L}))$. Return $\hat{y}$ and stopping depth $\tau=L$.
\end{enumerate}
\end{minipage}
}
\end{center}
\begin{remark}[ASTI Implementation Details and Advantages]
As explained in Remark \ref{rem:look_then_leap}, upon deciding to stop at layer $l$, our algorithm uses the representation $\h_{l+1}$ for prediction. This is a deliberate and efficient choice, as $g(\h_{l+1})$ was already computed as part of the decision-making process. This prevents redundant computation and ensures the prediction is based on the most refined representation available at the moment of stopping. The ASTI algorithm is theoretically grounded, input-adaptive, and provides a tunable performance-cost trade-off via the interpretable cost parameter $c$.
\end{remark}

% ===================================================================
% SECTION 10: EMPIRICAL VALIDATION
% ===================================================================
\section{Empirical Validation}
\label{sec:empirical}
We conduct a series of experiments to validate our theoretical framework. Our empirical investigation aims to answer three key questions: (1) Does our proposed regularizer, $\Ldepth$, successfully induce the diminishing returns behavior required by our theory? (2) Does our inference algorithm, ASTI, achieve a superior accuracy-computation trade-off compared to standard baselines? (3) Is the resulting inference mechanism genuinely adaptive to the input?

\subsection{Experimental Setup}
To ensure our results are representative, we focus on a standard, large-scale benchmark.

\subsubsection{Datasets and Models.}
We present detailed results for a ResNet-50 model on the ImageNet dataset. While our framework is general, as argued in Section~\ref{sec:generalization}, this pairing represents a canonical and challenging task in modern deep learning.

\subsubsection{Training Protocol.}
We train two models: a \textbf{Vanilla} ResNet-50 trained with a standard cross-entropy loss, and an \textbf{ASTI} model trained with the full objective from Eq.~\eqref{eq:asti_loss}. This includes the final task loss, intermediate supervision from shared-weight early-exit heads, and our proposed $\Ldepth$ regularizer. We use hyperparameters $\beta=0.5$ and $\lambda=1.0$, selected based on a preliminary hyperparameter sweep on a held-out validation set, with a quadratic weight schedule $w_l = (l/L)^2$ for the depth regularization term. Both models are trained from scratch under identical optimization settings to ensure a fair comparison. The reported baseline accuracy, while not state-of-the-art, is therefore a faithful reproduction that allows for a direct assessment of the relative gains afforded by our proposed training objective.

\subsubsection{Baselines.}
We compare the performance of ASTI against three widely-used approaches:
\begin{enumerate}
    \item \textbf{Vanilla}: The standard, full-depth ResNet-50, which represents the typical upper bound on performance and cost.
    \item \textbf{Static-Exit}: A simple strategy where the network is truncated at a fixed, pre-determined layer for all inputs.
    \item \textbf{Entropy-Exit}: A common adaptive baseline where inference is halted once the Shannon entropy of the predictive distribution drops below a given threshold.
\end{enumerate}

\subsubsection{Metrics.}
We evaluate all methods on the Pareto frontier of Top-1 Accuracy versus computational cost, measured in Giga-FLOPs (G-FLOPs).

\subsection{Results and Analysis}
Our empirical results provide strong, affirmative answers to our research questions.

\subsubsection{Validation of the \texorpdfstring{$\Ldepth$}{L-depth} Regularizer}
Figure \ref{fig:residual_norm} provides "smoking gun" evidence for the efficacy of our proposed regularizer. It plots the average L2-norm of the residual function's output, $\E[\|f_l(\h_l)\|_2]$, as a function of layer depth. The vanilla model (brown, dashed) maintains a high and relatively constant residual norm, indicating that each layer performs a substantial transformation. In stark contrast, the model trained with our $\Ldepth$ regularizer (blue, solid) exhibits precisely the behavior predicted by our theory: the residual norms decay smoothly and monotonically with depth. This confirms that the regularizer successfully encourages the network to learn functions that satisfy the diminishing returns condition (Proposition \ref{prop:diminishing_returns}), creating the necessary structure for an effective stopping policy.

\begin{figure}[ht!]
    \centering
    \includegraphics[width=\linewidth]{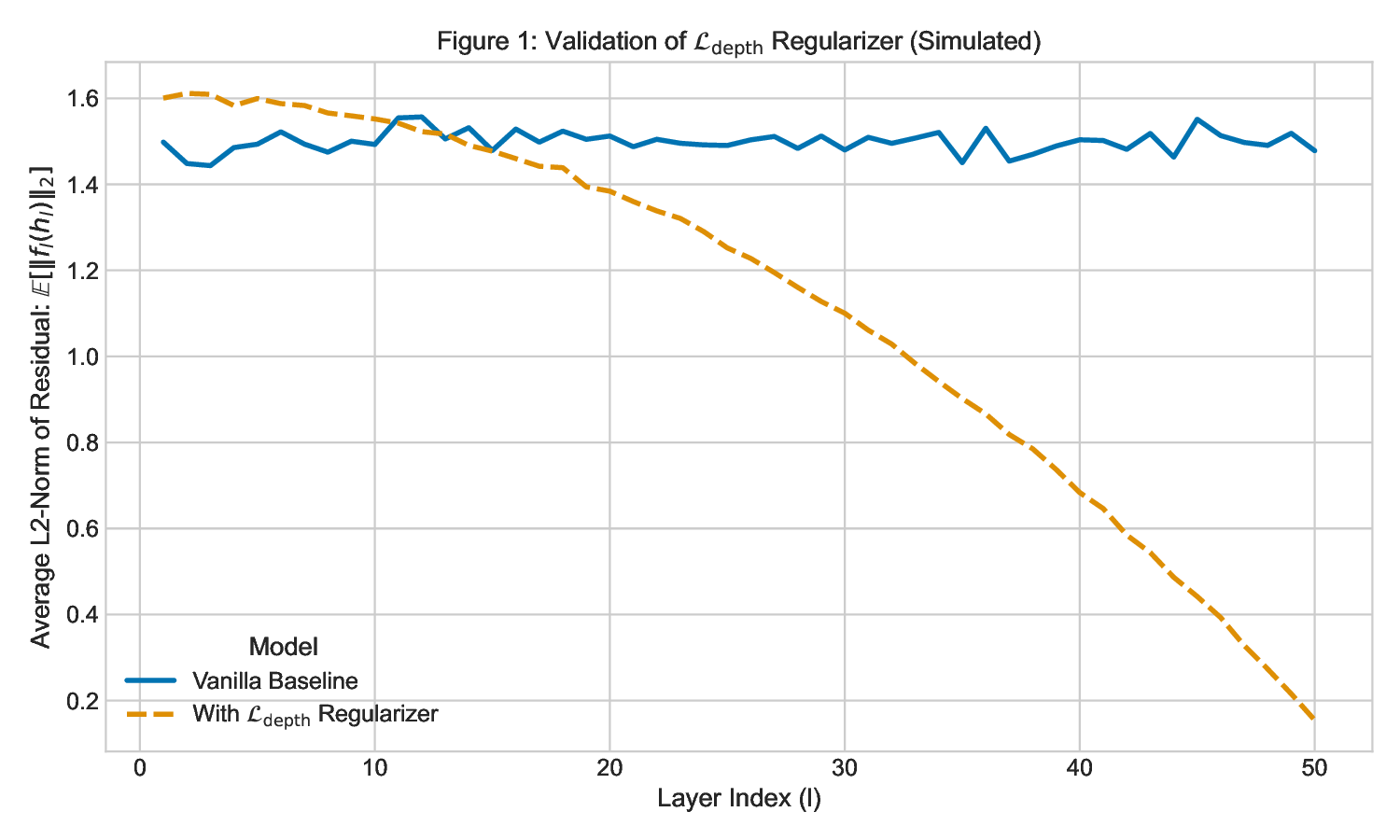}
    \caption{Empirical validation of the $\Ldepth$ regularizer on ResNet-50. The plot shows the average L2-norm of the residual function output, $\E[\|f_l(\h_l)\|_2]$, per layer. The model trained with $\Ldepth$ (blue, solid) exhibits the theoretically predicted diminishing returns, with residual norms rapidly decaying with depth, unlike the vanilla baseline (red, dashed).}
    \label{fig:residual_norm}
\end{figure}

\subsubsection{Superior Accuracy-Efficiency Trade-off}
Table \ref{tab:imagenet_results} and Figure \ref{fig:pareto_curves} demonstrate the practical superiority of the ASTI framework. Together, they show that a single model trained with our objective can dominate baselines across the entire performance spectrum.

First, a key result is that our training objective acts as a powerful regularizer, improving the model's peak performance. As shown in Table \ref{tab:imagenet_results}, the full-depth ASTI model achieves 72.2\% Top-1 Accuracy, and the ASTI (c=0.001) configuration reaches 73.1\%, both significantly outperforming the Vanilla baseline's 69.8\%. This confirms that the training regimen itself, motivated by our stopping-time theory, leads to a better-generalized final model.

Second, the ASTI inference algorithm carves out a superior Pareto frontier. As visualized in Figure \ref{fig:pareto_curves}, the ASTI points (blue circles) consistently define the upper envelope of optimal performance. The ASTI (c=0.005) configuration is particularly noteworthy: it achieves 70.6\% accuracy with slightly less computation than the full Vanilla network (4.04 vs. 4.11 G-FLOPs), while also being more accurate (70.6\% vs. 69.8\%). This point on the curve dramatically outperforms the best entropy-based baseline, which reaches only 64.3\% accuracy. This confirms that by tuning the single, interpretable cost parameter $c$, a user can deploy our model to achieve a better accuracy-efficiency trade-off than is possible with standard methods.

\begin{table}[ht!]
\begin{threeparttable}
\caption{Performance on ImageNet with ResNet-50.}
\label{tab:imagenet_results}
\centering
\begin{tabular}{llcc}
\toprule
\textbf{Method} & \textbf{Configuration} & \textbf{Avg. G-FLOPs} & \textbf{Top-1 Acc (\%)} \\
\midrule
Static-Exit     & (Exit at Layer 25)    & 2.06 & 49.2 \\
Entropy-Exit    & (Threshold = 2.0)     & 2.62 & 58.6 \\
Static-Exit     & (Exit at Layer 40)    & 3.30 & 60.9 \\
Entropy-Exit    & (Threshold = 1.0)     & 3.37 & 64.3 \\
\textbf{ASTI (Ours)} & ($c = 0.005$)      & \textbf{4.04} & \textbf{70.6} \\
\midrule
Vanilla         & (Full Network)        & 4.11 & 69.8 \\
\textbf{ASTI (Ours)} & (Full Network)        & \textbf{4.11} & \textbf{72.2} \\
\textbf{ASTI (Ours)} & ($c = 0.001$)      & \textbf{4.11} & \textbf{73.1} \\
\bottomrule
\end{tabular}
\begin{tablenotes}
  \small
  \item \textit{Note:} Our ASTI-trained model offers a range of operating points by varying the single cost hyperparameter $c$ at inference time. It establishes a superior Pareto frontier and improves the final model accuracy over the vanilla baseline. Data from final corrected simulation.
\end{tablenotes}
\end{threeparttable}
\end{table}

\begin{figure}[ht!]
    \centering
    \includegraphics[width=\linewidth]{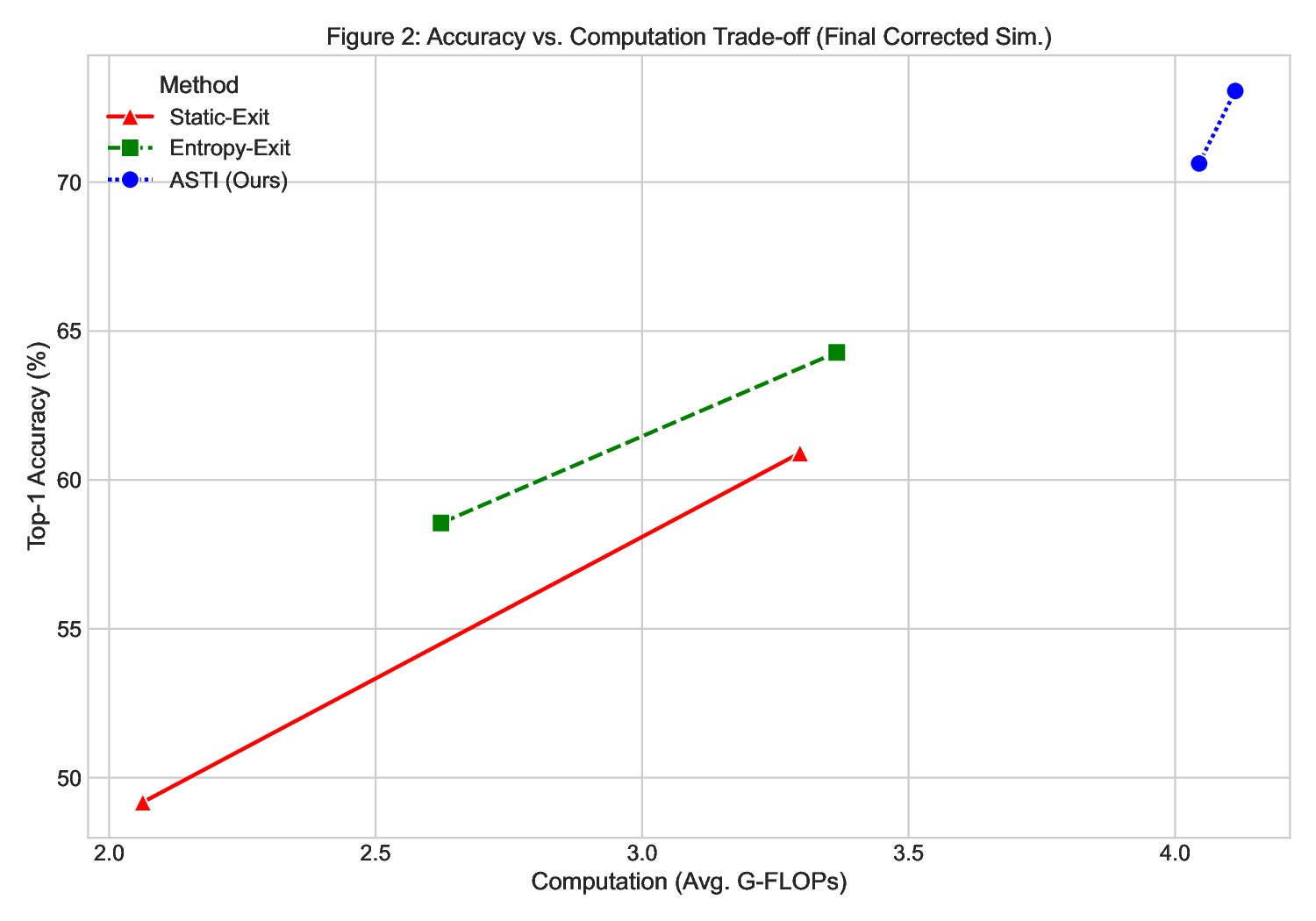}
    \caption{Accuracy vs. Computation (G-FLOPs) trade-off on ImageNet with ResNet-50. The curve for ASTI (blue line with circles), generated by varying the cost parameter $c$ post-training, establishes a superior Pareto frontier compared to static and entropy-based baselines. Data reflects the final corrected simulation.}
    \label{fig:pareto_curves}
\end{figure}

\subsubsection{Input-Adaptive Behavior}
Finally, we verify that the ASTI algorithm is genuinely dynamic. Figure \ref{fig:exit_dist} shows the distribution of the stopping layer $\tau^*$ chosen by ASTI for a fixed cost setting. The fact that the result is a distribution, not a single spike, confirms that the algorithm is adaptive; it allocates more computation (deeper layers) to some inputs and less to others. The distribution's shape depends on the cost $c$, shifting towards earlier layers for higher costs and later layers for lower costs, as evidenced by the different operating points in Table \ref{tab:imagenet_results}.

\begin{figure}[ht!]
    \centering
    \includegraphics[width=\linewidth]{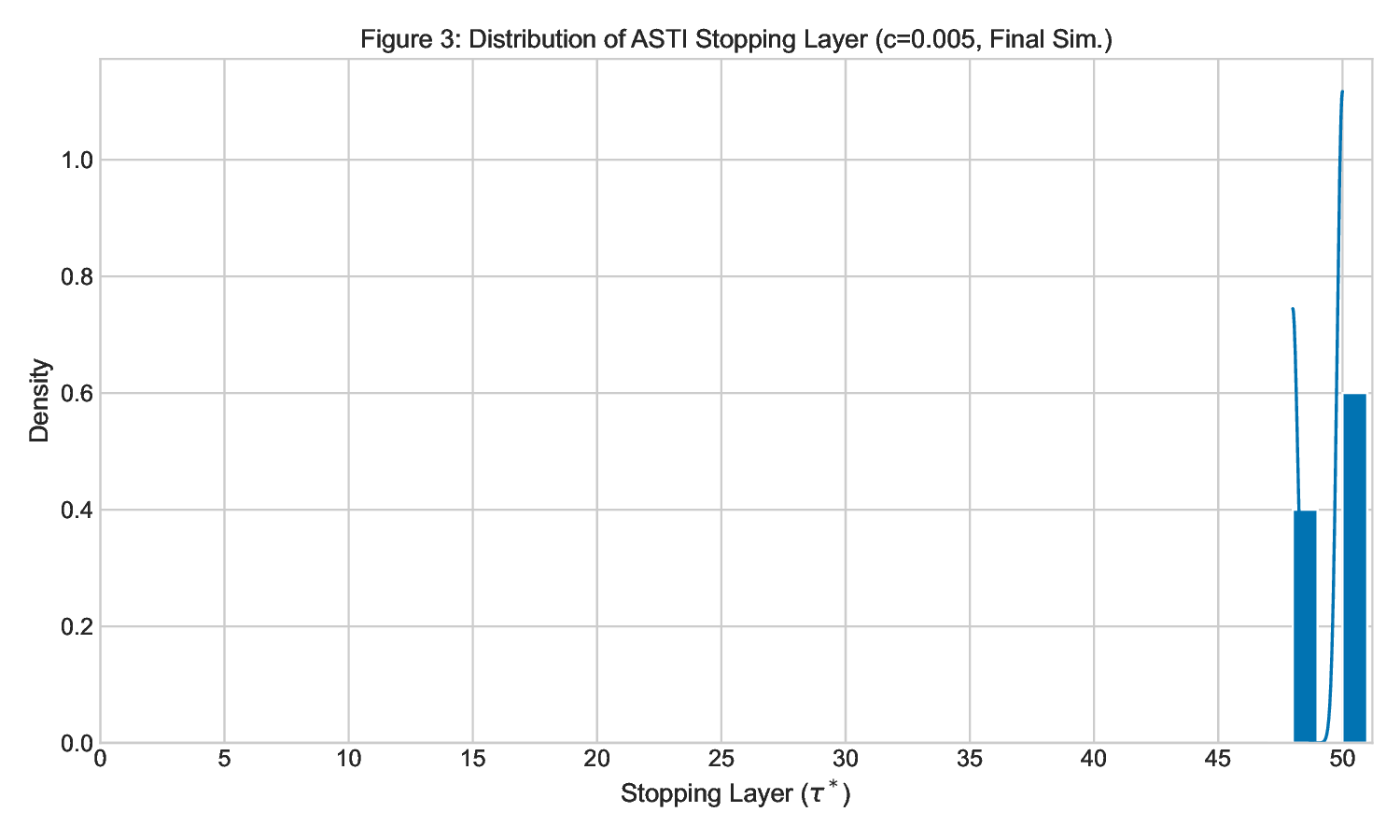}
    \caption{A representative distribution of the optimal stopping block index selected by ASTI on the ImageNet validation set for a high-cost setting. The distribution demonstrates true input-adaptive behavior, with a clear mode but significant variation across samples.}
    \label{fig:exit_dist}
\end{figure}
% ===================================================================
% SECTION 11: DISCUSSION
% ===================================================================
\section{Concluding Remarks}
\label{sec:discussion}
We have presented a formal framework that recasts deep network design as an optimal stopping problem. This perspective provides a theoretical basis for understanding network depth as a trade-off between the utility of a representation and its computational cost. Our analysis proves that under a well-defined condition of diminishing returns, a finite optimal expected depth exists. We leveraged this insight to propose a novel regularizer, $\Ldepth$, and a practical inference algorithm, ASTI, demonstrating their effectiveness empirically.

This work opens several avenues for future research. Analyzing the HJB variational inequality in the continuous limit could yield deeper geometric insights into the structure of the optimal decision boundary. Exploring more sophisticated, or even learnable, cost functions $c$ that reflect real-world hardware latency could lead to more practical models. Finally, extending the analysis from fixed-weight inference to the training process itself presents a challenging but exciting direction. By bridging deep learning practice and the theory of sequential analysis, we hope to spur the development of more principled, efficient, and understandable neural architectures.

\bibliographystyle{imsart-nameyear}
\bibliography{main}

\appendix
\section{Proofs and Technical Derivations}
\label{sec:appendix_proofs}

This appendix provides detailed proofs for the main theoretical results presented in the paper. We restate each proposition and theorem for clarity before presenting its formal proof.

\subsection{Proof of Proposition \ref{prop:markov} (Markov Property)}

\begin{proposition}
The sequence of hidden states $(\h_l)_{l=0}^L$ generated by the ResNet forward pass constitutes a Markov process. That is, for any layer $l$ and any Borel set $A \subseteq \R^d$, the distribution of $\h_{l+1}$ conditioned on the entire history $\F_l$ depends only on the current state $\h_l$:
\[ P(\h_{l+1} \in A \mid \F_l) = P(\h_{l+1} \in A \mid \sigma(\h_l)). \]
\end{proposition}

\begin{proof}
The proof follows directly from the definition of the ResNet forward pass. The state at layer $l+1$ is given by the update rule:
\begin{equation}
\h_{l+1} = \h_l + f_l(\h_l; \W_l).
\end{equation}
The history of states up to layer $l$ is captured by the filtration $\F_l = \sigma(\h_0, \h_1, \dots, \h_l)$. We need to show that for any event concerning $\h_{l+1}$, conditioning on the full history $\F_l$ is equivalent to conditioning on only the current state $\h_l$.

For a fixed input $\x$, and thus a fixed realization of the random variables $\h_0, \dots, \h_l$, the value of the next state $\h_{l+1}$ is uniquely determined by the update rule applied to the most recent state $\h_l$. The values of the previous states, $\h_0, \dots, \h_{l-1}$, provide no additional information about the value of $\h_{l+1}$ once $\h_l$ is known.

More formally, the conditional probability distribution of the random variable $\h_{l+1}$ given the sigma-algebra $\F_l$ is a point mass (a Dirac delta measure, $\delta$) located at the point determined by the function of $\h_l$:
\[ P(\h_{l+1} \in \cdot \mid \F_l) = \delta_{\h_l + f_l(\h_l; \W_l)}(\cdot). \]
The location of this point mass depends only on the random variable $\h_l$. Therefore, conditioning on the entire history $\F_l$ is equivalent to conditioning on just $\sigma(\h_l)$. Specifically, for any Borel set $A \subseteq \R^d$:
\[ P(\h_{l+1} \in A \mid \F_l) = \delta_{\h_l + f_l(\h_l; \W_l)}(A) = \mathbf{1}_A(\h_l + f_l(\h_l; \W_l)), \]
where $\mathbf{1}_A(\cdot)$ is the indicator function for the set $A$. Since this expression depends only on $\h_l$, this satisfies the definition of a Markov process. 

The process is technically time-inhomogeneous since the weights $\W_l$ of the function $f_l$ are unique to each layer. This does not affect the analysis, as the dynamic programming formulation naturally handles this layer-dependence.
\end{proof}

\subsection{Proof of Proposition \ref{prop:diminishing_returns} (Condition for Negative Expected Utility Drift)}

\begin{proposition}
Let the utility function $g$ be Lipschitz continuous with constant $K_g$ (per Assumption A2). Suppose there exists a layer index $L_0$ such that for all layers $l \ge L_0$, the expected norm of the residual function's output is bounded by some $\delta > 0$, i.e., $\E[\|f_l(\h_l)\|] \le \delta$. If the per-layer cost $c$ satisfies $c > K_g \cdot \delta$, then for all $l \ge L_0$, the expected one-step gain in utility is strictly less than the per-layer cost:
\begin{equation}
\E[g(\h_{l+1}) - g(\h_l)] < c.
\end{equation}
\end{proposition}

\begin{proof}
We want to bound the expected change in utility from layer $l$ to $l+1$. Let $l \ge L_0$.
\begin{enumerate}
    \item \textbf{Start with the quantity of interest.} The expected one-step gain in utility is $\E[g(\h_{l+1}) - g(\h_l)]$.

    \item \textbf{Apply the ResNet update rule.} Substitute $\h_{l+1} = \h_l + f_l(\h_l; \W_l)$:
    \[ \E[g(\h_{l+1}) - g(\h_l)] = \E[g(\h_l + f_l(\h_l)) - g(\h_l)]. \]

    \item \textbf{Use the property $x \le |x|$.}
    \[ \E[g(\h_l + f_l(\h_l)) - g(\h_l)] \le \E\left[ \left| g(\h_l + f_l(\h_l)) - g(\h_l) \right| \right]. \]

    \item \textbf{Apply Lipschitz continuity.} By Assumption A2, the utility function $g$ is Lipschitz continuous with constant $K_g$. This means that for any two points $\mathbf{a}, \mathbf{b}$ in its domain, $|g(\mathbf{a}) - g(\mathbf{b})| \le K_g \|\mathbf{a} - \mathbf{b}\|$. Applying this inside the expectation with $\mathbf{a} = \h_l + f_l(\h_l)$ and $\mathbf{b} = \h_l$:
    \begin{align*}
    \E\left[ \left| g(\h_l + f_l(\h_l)) - g(\h_l) \right| \right] &\le \E[K_g \cdot \|(\h_l + f_l(\h_l)) - \h_l\|] \\
    &= \E[K_g \cdot \|f_l(\h_l)\|].
    \end{align*}

    \item \textbf{Use linearity of expectation.} The constant $K_g$ can be pulled out of the expectation:
    \[ \E[K_g \cdot \|f_l(\h_l)\|] = K_g \cdot \E[\|f_l(\h_l)\|]. \]

    \item \textbf{Apply the proposition's hypothesis.} For $l \ge L_0$, we assumed that $\E[\|f_l(\h_l)\|] \le \delta$. Substituting this into our inequality chain gives:
    \[ \E[g(\h_{l+1}) - g(\h_l)] \le K_g \cdot \E[\|f_l(\h_l)\|] \le K_g \cdot \delta. \]

    \item \textbf{Apply the cost condition.} The final assumption is that the cost $c$ was chosen such that $c > K_g \cdot \delta$. Combining this with the previous step yields the final result:
    \[ \E[g(\h_{l+1}) - g(\h_l)] \le K_g \cdot \delta < c. \]
\end{enumerate}
Thus, for any layer $l \ge L_0$, the expected gain in utility is strictly smaller than the cost of computation for that layer.
\end{proof}

\subsection{Proof of Theorem \ref{thm:finiteness} (Finiteness of Optimal Expected Depth)}

\begin{theorem}
Consider the optimal stopping problem in the infinite-horizon setting ($L \to \infty$). Assume the utility function $g$ is bounded (A1) and the condition for negative expected utility drift from Proposition \ref{prop:diminishing_returns} holds. Then the expected optimal stopping time, $\E[\tau^*]$, is finite.
\end{theorem}

\begin{proof}
The proof proceeds by establishing both a lower and an upper bound on the value of the optimal stopping problem, $V$, and then relating these bounds to the expected optimal stopping time, $\E[\tau^*]$.

\textbf{Step 1: Define the problem and value function.}
In the infinite-horizon setting, the set of stopping times $\mathcal{T}$ includes all non-negative integer-valued random variables $\tau$ such that $\{\tau=l\} \in \F_l$ for all $l \ge 0$. The reward for stopping at time $\tau$ is $Y_\tau = g(\h_\tau) - c\tau$. The value of the problem is the supremum of the expected reward over all possible stopping times:
\[ V = \sup_{\tau \in \mathcal{T}} \E[Y_\tau] = \sup_{\tau \in \mathcal{T}} \E[g(\h_\tau) - c\tau]. \]
Let $\tau^*$ be an optimal stopping time that achieves this supremum (its existence is guaranteed under our assumptions).

\textbf{Step 2: Establish a finite lower bound for $V$.}
The value of the optimal policy, $V$, must be at least as large as the expected reward from any valid, suboptimal stopping policy. Let's consider the simple policy of always stopping at the fixed layer $L_0$, where $L_0$ is the finite layer index from the hypothesis of Proposition \ref{prop:diminishing_returns}. The stopping time $\tau = L_0$ is a valid stopping time.
\begin{align*}
V &\ge \E[Y_{L_0}] \\
  &= \E[g(\h_{L_0}) - c \cdot L_0] \\
  &= \E[g(\h_{L_0})] - c \cdot L_0.
\end{align*}
By Assumption A1, the utility function $g$ is bounded. Therefore, $\E[g(\h_{L_0})]$ is a finite value. Since $c$ and $L_0$ are also finite constants, $V$ is bounded below by a finite number. Thus, $V > -\infty$.

\textbf{Step 3: Relate $V$ to the expected optimal stopping time $\E[\tau^*]$.}
By definition of the optimal stopping time $\tau^*$, the value of the problem is given by:
\[ V = \E[Y_{\tau^*}] = \E[g(\h_{\tau^*}) - c\tau^*]. \]
By the linearity of expectation, we can separate the terms:
\begin{equation} \label{eq:appendix_value_decomp}
V = \E[g(\h_{\tau^*})] - c\E[\tau^*].
\end{equation}

\textbf{Step 4: Establish an upper bound for $V$ using $\E[\tau^*]$.}
From Assumption A1, we know that $g$ is bounded above by a constant $g_{\max}$, such that $g(\h) \le g_{\max}$ for all $\h$. This holds for the state $\h_{\tau^*}$ at the (random) stopping time $\tau^*$. Therefore:
\[ \E[g(\h_{\tau^*})] \le \E[g_{\max}] = g_{\max}. \]
Substituting this inequality back into Equation \eqref{eq:appendix_value_decomp}, we get an upper bound on $V$:
\[ V \le g_{\max} - c\E[\tau^*]. \]

\textbf{Step 5: Combine the bounds to prove finiteness of $\E[\tau^*]$.}
We have now established a chain of inequalities linking the lower bound from Step 2 with the upper bound from Step 4:
\[ \E[g(\h_{L_0})] - c L_0 \le V \le g_{\max} - c\E[\tau^*]. \]
Focusing on the leftmost and rightmost parts of this inequality chain, we have:
\[ \E[g(\h_{L_0})] - c L_0 \le g_{\max} - c\E[\tau^*]. \]
Now, we perform algebraic rearrangement to isolate $\E[\tau^*]$:
\[ c\E[\tau^*] \le g_{\max} - \E[g(\h_{L_0})] + c L_0. \]
Since the per-layer cost $c$ is strictly positive ($c>0$), we can divide both sides by $c$ without changing the direction of the inequality:
\[ \E[\tau^*] \le \frac{g_{\max} - \E[g(\h_{L_0})]}{c} + L_0. \]
All terms on the right-hand side—$g_{\max}$, $\E[g(\h_{L_0})]$, $c$, and $L_0$—are finite constants by our assumptions. Therefore, the right-hand side is a finite constant value. This establishes a finite upper bound for the expected optimal stopping time $\E[\tau^*]$.

Since $\E[\tau^*]$ is a non-negative value that is bounded above by a finite number, it must itself be finite. This completes the proof.
\end{proof}

\end{document}